\documentclass[runningheads]{llncs}
\usepackage{graphicx} 
\usepackage{algorithm}
\usepackage{enumitem}
\usepackage{xcolor}
\usepackage{amsmath}
\usepackage{amssymb}
\usepackage{mathtools}
\usepackage[noend]{algpseudocode}
\usepackage{booktabs}
\usepackage{xspace}
\usepackage{float}
\usepackage{orcidlink}
\usepackage{listings}
\usepackage{pgfplots}
\usepackage{pgfplotstable}
\usepackage{subcaption}

\usepgfplotslibrary{groupplots}
\pgfplotsset{compat=newest}
\usetikzlibrary{pgfplots.statistics,shapes,arrows,positioning,fit,backgrounds,matrix,chains,scopes,calc}

\newcommand{\ourevaluator}{OVAL\xspace}
\newcommand{\ghentEvaluator}{SolidLab Evaluator\xspace}
\newcommand{\ghentEvaluatorFirst}{SolidLab ODRL Evaluator\xspace}

\newcommand{\policy}[3]{\langle \mathbf{#1}, \mathbf{#2}, \mathbf{#3} \rangle}

\newcommand{\fullpolicy}[7]{\langle \mathbf{#1}, \mathbf{#2}, \mathbf{#3}, \mathbf{#4}, \mathbf{#5}, \mathbf{#6}, \mathbf{#7} \rangle}

\newcommand{\match}[2]{\mathsf{match}(#1,#2)}

\AtBeginDocument{%
  }

\author{
Jaime Osvaldo Salas\orcidlink{0000-0002-9353-8955} \and
Paolo Pareti\orcidlink{0000-0002-2502-0011} \and
Adeel Aslam \and
Christopher Maidens \orcidlink{0000-0002-4385-7202} \and
George Konstantinidis\orcidlink{0000-0002-3962-9303}}

\institute{University of Southampton, University Road, Southampton, SO17 1BJ, UK \\ Contact: \email{j.o.salas@soton.ac.uk}, \email{p.pareti@soton.ac.uk}}

\authorrunning{Salas et al.}

\lstdefinestyle{ODRL}{
basicstyle=\ttfamily, 
mathescape, 
columns=fixed,
fontadjust=true,
basewidth=0.5em
}

\begin{document}

\title{A Formally Grounded ODRL Evaluator: Implementation and Comparison}
\date{May 2026}

\maketitle

\begin{abstract}

The ODRL policy language is emerging as the de-facto standard for policy modelling data access and usage preferences, AI governance policies and data workflows in European dataspaces. The current standard has no mathematical formal semantics to describe how a system should implement policy evaluation. This has resulted in a variety of systems and tools that implement their own interpretation of the language, which limits interoperability and cannot guarantee consistent results. Based on an existing semantic model of ODRL, we formalise the problems of ODRL evaluation for the access control and monitoring scenarios, in both static and streaming settings, and we provide a novel, efficient algorithm and implementation. We present the first ODRL Evaluator with transparent formal semantics and supporting all rule types. We experimentally measure its performance, analysing different scalability dimensions related to policy complexity and size of the data on which a policy is evaluated. We compare our system with the state-of-the-art by providing a comparative review of existing ODRL evaluators, which highlights the differences in supported ODRL features and evaluation modes.

\end{abstract}

\section{Introduction}

The Open Digital Rights Language (ODRL) is a W3C recommendation, currently in version 2.2, that provides a data model and a base vocabulary for the definition of computational policies. Based on the RDF data model, and an OWL vocabulary, ODRL is a natively Semantic Web technology. 
While many of its features can be seen as domain agnostic, such as the ability to define agreements, and the deontic rules of permissions, prohibitions and obligations, it is primarily designed to regulate the use of digital content and services. ODRL has recently gained attention as the de-facto policy language to regulate the use of data by AI systems and compliance-focussed initiatives like the European Dataspaces \cite{yi2025compliance,golpayegani2024aiup,petrova2025knowledge}. 

Since its inception, several studies have attempted to define a formal semantics for ODRL, namely a precise definition of how ODRL policies should be interpreted. Most notably, Pucella and Weissman formalised ODRL 1.1 to logical formulas~\cite{pucella2006formal}, Steyskal and Polleres~\cite{steyskal2015towards} defined semantics for ODRL 2.0 based on rule matching and logic programs, Bonatti et al.~\cite{bonatti2025towards} proposed declarative semantics for ODRL 2.2, and Salas et al.~\cite{salas2025evaluation} defined formal semantics for ODRL 2.2. based on query answering. A community report on ODRL semantics is currently being worked on, but it does not provide a formal semantics in the mathematical sense.\footnote{\url{https://w3c.github.io/odrl/formal-semantics/}} In this paper, we reuse the semantics based on query answering defined by Salas et al., as it is the most recent and comprehensive semantics based on a mathematical model. 

Central to understanding ODRL semantics is the concept of an \emph{ODRL Evaluator}, defined in the ODRL specification as ``a system that determines whether the rules of an ODRL policy meet their intended action performance''~\cite{w3c2018ODRLInfoModel}.
Two scenarios that naturally occur in the context of ODRL evaluation are \emph{monitoring} and \emph{access control}, both of which have been studied in the literature \cite{salas2025evaluation,slabbinck2025interoperable} and are described in the community report on ODRL semantics. While monitoring typically refers to evaluating a policy on a \emph{state of the world}, intuitively a log of past actions, to see if any violation occurred, access control evaluates policies primarily on access requests, which describe an intent to perform future actions. 
In this article we use the term \emph{ODRL Evaluator} for a system that addresses one or both of the monitoring and access control scenarios, but note that the terms \emph{ODRL Monitor} and \emph{ODRL Authoriser} have been suggested to specify whether an ODRL processor targets the former or the latter \cite{rodriguez2025towards}. 
A number of systems have been developed to evaluate ODRL. However, most only cover a limited subset of the ODRL language, lack formal semantics, and are based on a partial, and sometimes unspecified, interpretation of the evaluation problem.

Motivated by this gap in the state-of-the-art, our first contribution is an extension of the ODRL semantics of Salas et al.~\cite{salas2025evaluation} that models the decision problems for both policy evaluation scenarios in a single framework and shows how these two problems coincide. Moreover, to ensure the applicability of ODRL evaluation in a realistic setting, we extend the monitoring scenario to a streaming (online) mode of evaluation, where the state of the world changes frequently.

Our second contribution is a novel algorithm for ODRL evaluation that covers all of the ODRL features in the chosen semantics. Notably, this algorithm performs a single-pass evaluation over the state of the world, allowing it to efficiently process incoming streams of events and access requests in constant time with respect to the size of the state of the world. 

As our third contribution we present \ourevaluator, a new open source ODRL Evaluator that provides a full implementation of this algorithm. We provide an analysis of the performance of this system, showing its scalability with respect to the size of the inputs. 
Our fourth main contribution is to review existing ODRL evaluators and provide a structured comparison against \ourevaluator based on features. For the most mature of such systems, we also assess performance relative to our system.

Our contributions provide a formal underpinning to ODRL evaluation, moving it from an underspecified and ad-hoc implementation task to a standardised and transparent method, proving clear connections between access control, offline monitoring, and online monitoring within a single semantic framework. Our implementation delivers a scalable system: OVAL exhibits linear growth with respect to both policy size and number of events, and produces results within milliseconds for existing datasets being up to two orders of magnitude more efficient than state-of-the-art.

\section{Preliminaries} \label{sec:preliminaries}

We now briefly introduce core concepts of ODRL semantics, and refer the reader to Salas et al.~\cite{salas2025evaluation} for more details. As the main building blocks of a policy, ODRL \emph{rules} are sets of \emph{constraints}, which in turn are boolean combinations of \textit{conditions} of the form $(\lambda \ op\ \rho)$. The left operand $\lambda$ is a property that is drawn from a known vocabulary, or Knowledge Graph, whose value is compared to the right operand $\rho$ by the operator $op$. The operator $op$ can be a binary operator between constants ($=$, $>$, $\geq$, $<$, $\leq$, $\neq$), a binary operator between sets ($\equiv$, $\in$, $\notin$, $\supset$, $\subset$), or a special operator that checks class membership ($=_{type}$). 

These constraints are evaluated on sets of \textit{events}, denoted \textit{states of the world}, where each event $e$ is a tuple of fixed length, whose indexes represent specific ODRL core components and/or left operands and thus allow the description of the event with respect to those features. For simplicity, we assume that the first value $t$ of event $e$ is always the timestamp of when the action has been performed, followed by the action value. With the exception of the indexes representing the timestamp and action, events can have null values. For example, assume a mapping between the first five indexes of an event and the features Party, Action, Asset and Page, respectively. Under this mapping, events $e_{Alice}: \;\langle \text{``2026-10-10T10:02:00''},Alice, Read, Book, 300\rangle$ and $e_{Bob}:\; \langle \text{``2026-10-10T10:01:00''}, Bob, Edit, Book,null \rangle$ represent, intuitively, the fact that Alice read page 300 of Book at datetime ``2026-10-10T10:02:00'', and Bob edited it on ``2026-10-10T10:01:00'' (pages not specified), and a state of the world $\omega=\{ e_{Alice}, e_{Bob}\}$ denotes a record that both events have occurred. 

We say that a rule $R$ \emph{matches} an event $e$, denoted by $\mathsf{match}(R, e)$ if all its constraints evaluate to true on $e$, that is when $(\bigwedge_{c \in R}c)(e)$ is true. Thus, when evaluating a rule, all of its constraints are conjuncted with each other. 

ODRL policies are usually read as \emph{\underline{Rule} for \underline{Party} to \underline{Action} \underline{Asset} subject to \underline{Constraints} and \underline{Refinements}}, \emph{e.g.} \emph{Permission for Alice to Read a Book subject to the constraint the date is less than December 12th 2026 and the Book to be refined to its Subset of Pages 1-400}. 
Such a rule could be formalised as the following set of constraints: $p_{Read}: \; \{ (Party = Alice) , (Action = Read) , (Asset = Book) , (Date < \text{``2026-12-12T00:00:00''}) , (Pages \geq 1) , (Pages \leq 400) \}$. 

An \emph{ODRL Lite policy} is a tuple $\policy{P}{F}{O}$, where $\mathbf{P}$, $\mathbf{F}$ and $\mathbf{O}$ are the sets of permission, prohibition and obligation rules, respectively. A state of the world $\omega$ is said to be \emph{valid} with respect to a policy $\policy{P}{F}{O}$ if the following are all true:
\begin{itemize}
    \item for all events $e \in \omega$, there exists a permission $P \in \mathbf{P}$ s.t. $\match{P}{e}$ is true.
    \item for all events $e \in \omega$, there is no prohibition $F \in \mathbf{F}$ s.t. $\match{F}{e}$ is true.
    \item for all obligations $O \in \mathbf{O}$, there exists an event $e \in \omega$ s.t. $\match{O}{e}$ is true.
\end{itemize}

ODRL 2.2 also allows for the definition of duties, remedies and consequences, which are obligations associated to permissions, prohibitions and obligations, respectively. A full ODRL policy is a 7-tuple $\Pi = \fullpolicy{P}{F}{O}{DP}{DPC}{FR}{OC}$ where $\mathbf{DP}$, $\mathbf{FR}$ and $\mathbf{OC}$ are sets of pairs of permissions and duties, prohibitions and remedies, and obligations and consequences, respectively, and $\mathbf{DPC}$ is a set of triples of permissions, duties and consequences. In such a case, a state of the world $\omega$ is valid w.r.t. a full ODRL policy if it is \textit{valid} w.r.t. $\policy{P}{F}{O}$ and: 
\begin{itemize}
    \item for all permission-duty pairs $\langle D,P \rangle \in \mathbf{DP}$ and for all events $e \in \omega$, if $\match{P}{e}$ is true, then there exists an event $e' \in \omega$ that occurred before $e$ ($t' \leq t$) such that $\match{D}{e'}$ is true, and
    \item for all prohibition-remedy pairs $\langle F,R \rangle \in \mathbf{FR}$ and for all events $e \in \omega$, if $\match{F}{e}$ is true, then there exists an event $e' \in \omega$ that occurred after $e$ ($t' \geq t$) such that $\match{R}{e'}$ is true, and
    \item for all obligation-consequence pairs $\langle O,C \rangle \in \mathbf{OC}$ such that $O$ contains a time constraint $c_t=(Datetime \leq t)$, if there is no event $e$ such that $\match{O}{e}$ is true, then there must exist events $e'$ and $e''$ that occurred after $e$ ($t \leq t' \wedge t \leq t'')$ such that $\match{C}{e'}$ and $\match{O\setminus\{ c_t\}}{e''}$ are true,
    \item for all permission-duty-consequence triples $\langle D,P,C \rangle \in \mathbf{DPC}$ such that $D$ contains a time constraint $c_t=(Datetime \leq t)$, if there exists an event $e$ such that $\match{P}{e}$ is true and there is no event $e'$ such that $\match{D}{e'}$ is true, then there must exist subsequent events $e''$ and $e'''$ such that $\match{C}{e''}$ and $\match{D\setminus\{c_t\}}{e'''}$ are true.
\end{itemize}

In our running example, state of the world $\{e_{Alice},e_{Bob}\}$ would violate a policy that has permission $p_{Read}$ as the only rule, because the event $e_{Bob}$ is not explicitly permitted, and thus it is implicitly prohibited. Note that the chosen semantics does not include reasoning, that is the inference of new facts, and constants are assumed to be disjoint from each other, following the unique name assumption.

\section{Related Work}

\subsection{Formalisations}
The earliest version of ODRL (1.0) was tailored for digital rights management, whereas currently it has been adopted for more general usage including the usage of content and services~\cite{iannella2018odrl}. 
ODRL 2.0 added prohibitions and duties as rule types, resembling the concepts of deontic logic more closely.
The current specification of ODRL (2.2) was formalised by Iannella and Villata~\cite{iannella2018odrl}. Notably, it builds on ODRL 2.0 by adding additional temporal or conditional rules (e.g., consequences and remedies). Finally, constraints were structured to triples comprised of a left operand, an operator and a right operand.

We now review the most relevant pieces of work in the literature that address similar semantics problems. Garc{\'\i}a et al.~\cite{garcia2005formalising} employed semantic web technologies such as OWL~\cite{mcguinness2004owl}, to translate ODRL (in XML format) to OWL ontologies and RDF graphs, which allowed them to leverage the IPROnto ontology for digital rights management~\cite{delgado2003ipronto}. However, their approach was tailored for ODRL 1.1, which is deprecated, and lacks the richer tapestry of rule types such as prohibitions, and obligation. 
Pucella and Weissman~\cite{pucella2006formal} proposed a translation of ODRL 1.1 Policies to logical formulas. The normalisation steps discussed in this study could also be applied to their constraints, which would also allow logical normal forms.  
Steyskal and Polleres~\cite{steyskal2015towards} earliest proposed semantics for ODRL 2.0+ based on rule-matching and logic programs; notably their semantics took into consideration the implicit dependencies between actions. 
Bonatti et al.~\cite{bonatti2025towards} defined declarative semantics for ODRL 2.2 that covered the main components of the language, while also defining the notion of compliance with a policy, providing illustrative examples with complex policies that include duties, remedies and consequences. 
Salas et al.~\cite{salas2025evaluation} defined formal semantics for ODRL 2.2 inspired by query evaluation over a state of the world that acts as a database. Moreover, they introduced the notion of policy comparison inspired by query containment.
Slabbinck et al.~\cite{slabbinck2025interoperable} defined a vocabulary to describe interoperable ODRL Evaluation reports and also presented the \ghentEvaluatorFirst, an ODRL Evaluator for access control scenarios. 
Kebede et al.\cite{kebede2018critical} discuss the representational power of ODRL, reinforced by presenting use-cases and examples, and highlights some of the limitations of ODRL. For example, the ambiguous semantics of duties (as per the currently published recommendations) or the granularity of parties (e.g., certain members of an organisation). 
De Vos et al.~\cite{de2019odrl} implement policy and compliance checking by translating policies into answer-set programs to check for compliance. The semantics of answer-set programming are ideal for applications that adopt a closed-world assumption. Their proposed model is ideal for representing real-world legal frameworks such as the GDPR. 
Kieffer et al.~\cite{kieffer2025composing} propose an approach that computes the least restrictive license that complies to a set of policies. If such a license does not exist, one can assume that there are no events that are valid w.r.t. all policies simultaneously. Therefore, said approach could be used to verify policy overlap. However, the problems of containment and equivalence were not a focus of their work, and cannot be trivially derived.

\subsection{ODRL Evaluation Systems}

The \ghentEvaluator ~\cite{slabbinck2025interoperable} leverages the various standards defined by the W3C ODRL Community Group such as ODRL Evaluators and Compliance Reports as defined per the ODRL Information Model\footnote{ODRL Information Model 2.2: https://www.w3.org/TR/odrl-model/\#terminology} 
and supports reasoning tasks when evaluating policies using the EYE~\footnote{EYE reasoner: https://github.com/eyereasoner/eye} reasoner. Its documentation specifies that the only rule types that are currently supported are Permissions and Prohibitions, with a partial support of Duties through the State of the World object.\footnote{\url{https://github.com/SolidLabResearch/ODRL-Evaluator/blob/main/ODRL-Support.md}, accessed 13/07/2026}

The ODRE Framework~\cite{cimmino2024opendigitalrightsenforcement} is a framework designed to provide enforcement capabilities to ODRL policies. They achieve this by adding more features to ODRL, creating layers with additional expressivity, which combine different technologies to enforce the features. For example, they add the ability to define data variables using the interpolated language Freemarker~\footnote{https://freemarker.apache.org/docs/dgui\_misc\_alternativesyntax.html}. As a result, they produce enforceable policies through Python and Java. With regards to policy evaluation specifically, it supports only the \texttt{odrl:dateTime} left operand, and only numeric and symbolic ODRL operators~\cite{salas_trejo_2026_18862527}. 
Due to lack of documentation and problems in running the code, we were unable to test their system. However, we could access a demonstration that illustrated its functioning.

ODRL-PAP\footnote{ODRL-PAP: https://github.com/wistefan/odrl-pap} is a mostly Java implementation that translates ODRL policies to equivalent rego\footnote{Open Policy Agent: https://www.openpolicyagent.org/docs/policy-language} expressions which are then fed into an Open Policy Agent (OPA). The result of the evaluation is given as input to the Policy Enforcement Point, which the user interacts with, and can actually enforce the policy. However, the implementation seemingly only takes into consideration a limited variety of constraints, tailored for a curated list of vocabularies (based on concepts from the DOME-Marketplace project\footnote{DOME-Marketplace project: https://dome-marketplace.eu/}). For example, for ODRL, there are only translations for constraints with \texttt{odrl:dayOfWeek}, \texttt{odrl:hourOfDay} and \texttt{odrl:currentTime}.

The ODRL Manager~\footnote{ODRL Manager: https://github.com/Prometheus-X-association/odrl-manager} is described as an ODRL Library for Node.js with functions to facilitate the parsing, validation and interpretation of ODRL 2.2 Policies. Upon a preliminary analysis, it appears to have the infrastructure to support ODRL Evaluation. However, we were unable to find a publication that formalises their semantics, and only accepts ODRL policies in JSON format with little space for deviations in notation.

The MOSAICrOWN \footnote{MOSAICrOWNpolicy engine: https://github.com/mosaicrown/policy-engine} is a Python implementation of a policy engine that aligns with Salas et al.'s semantics~\cite{salas2025evaluation} insofar as ODRL policies that mention targets, assignees, actions and purposes. This makes sense since the implementation relies on a SPARQL query to traverse the graphs that represent the ODRL policies considering that the semantics are based on query evaluation. Access requests are modelled as SQL query requests. 

Gaia-X transforms contracts, policies and legislation into ODRL and performs reasoning using their Policy Reasoning Engine \footnote{GAIA-X Across  [ https://gitlab.com/gaia-x/lab/policy-reasoning/odrl-vc-profile ]  [ https://docs.gaia-x.eu/technical-committee/architecture-document/24.04/ Compliance and Sovereign Data Exchange work packages ] }
Their engine requires a set of policies (expressed using Verifiable Credentials\cite{w3c_vc_data_model_2_0}), information about the consumer such as their location, and the consumer's usage intentions expressed as an ODRL Request. Their policies employ alternative constraints to those of the ODRL standard, where the left operands are formulated as JSON Path queries. The policies are translated into RDF triples and loaded into a GraphDB instance, after which SPARQL queries are executed over the graph. If the queries return any results, then the system concludes that an agreement can be reached. This appears to align with the ODRL Lite semantics, but we were unable to find publications that describe their evaluation semantics formally to confirm this.

\section{ODRL Evaluation Problems}\label{sec:evaluatingScenarios}

In this section we formalise the monitoring and access control scenarios into a single semantic framework, and show under which conditions the chosen semantics allows translations from one to the other. 
Moreover, we define the monitoring scenario also under another dimension: streaming vs static states of the world.

\subsection{Monitoring}

The monitoring problem asks the question, given a compilation of events that have occurred (a state of the world) and an ODRL policy, whether the state of the world is valid with respect to the policy.  As is normally the case for monitoring scenarios, the state of the world $\omega$ corresponds to a database containing a set of ordered events that have been logged by an external agent. While there is no strict data model that such log should adhere to, for simplicity we will assume that it is a table, with each row representing a distinct event, and each column representing a feature of these events. Given a policy, we assume there is a partial mapping between each column of the log and the features, such  as the core components or the left operands, of the policy. We also assume that the first value in any event is the timestamp at which it occurred. If a column is not mapped to any feature of a policy, the information in that column is not used by the policy for evaluation. If a policy contains a feature that is not mapped to a column of the table, we assume that the value for that feature for any event contained in the policy is \emph{null}.

\begin{table}[]
    \centering
      \caption{Example of a state of the world}
    \begin{tabular}{llllll}
    \toprule
    Datetime & Action & Actor & Asset & Resolution & Pages \\
    \midrule
    2026-10-10T10:00:00\phantom{aaa} & Print\phantom{aaa} & Alice\phantom{aaa} & Book\phantom{aaa} & 500 dpi\phantom{aaaaa} & 300  \\
    2026-10-10T10:01:00 & Edit & Bob & Book & null & null \\
    2026-10-10T10:02:00 & Read & Alice & Book & null & 300 \\
    2026-10-11T09:00:00 & Print & Bob & Picture & 650dpi & null \\
    \bottomrule
    \end{tabular}
  
    \label{tab:sotw1}
\end{table}

Table \ref{tab:sotw1} shows an example of a state of the world. This state of the world contains four events about Alice and Bob performing actions on certain assets (a book and a picture). The property \textit{Resolution} is a constraint associated with the \textit{Print} action, while \textit{Pages} is a constraint associated with the \textit{Book} asset.

We further distinguish two modes of monitoring evaluation. In log-based (offline) monitoring evaluation, a policy is evaluated against a single state of the world object. This  corresponds directly to the definition of validity in Section \ref{sec:preliminaries}. 
\begin{definition}
    Determining whether a state of the world is valid with respect to a given ODRL policy is the offline monitoring evaluation decision problem. 
\end{definition}

As an example, let $\Pi$ be a policy with $\mathbf{P}=\{p_{Read}\}$ as a permission and  $\mathbf{DP}=\{\langle d_{Print}, p_{Read} \rangle\}$ a permission-duty pair where $d_{Print}$ is the rule $\{(Party=Alice), (Action=Print), (Asset=Book), (Resolution=500dpi), (Pages>200) \}$. This policy intuitively states that Alice is allowed to Read a Book if she has Printed the Book at a resolution of 500dpi and at least 200 pages. Under the monitoring scenario, the first event in Table \ref{tab:sotw1} will be matched by $d_{Print}$, and the third event will be matched by $p_{Read}$. Moreover, since the first event occurred before the third event, the permission-duty pair is satisfied. 

In stream-based (or online) monitoring, a policy is evaluated continuously against a stream of events $\overline\omega$, which is a time-ordered list of state of the worlds. Stream $\overline\omega$ is time-ordered if, for every event $e$ in batch $\overline\omega_i$, the time event $e$ is recorded to have happened is greater than the time recorded of every event $e$ in any batch $\overline\omega_j$ such that $j < i$. 
Intuitively, each state of the world in $\overline\omega$ could contain as little as a single event. Evaluating a policy over a stream of events should lead to the same result as evaluating the whole stream offline in a single run.

\begin{definition} \label{def:streaming}
A stream of events $\overline\omega$ is valid w.r.t.\ an ODRL policy $\Pi$ if and only if the state of the world $\bigcup_{\omega \in \overline\omega} \omega$ is valid w.r.t.\ $\Pi$.
\end{definition}

\begin{definition}
    Determining whether a stream of events $\overline\omega$ is valid w.r.t.\ an ODRL policy is the online monitoring evaluation decision problem. 
\end{definition}

\subsection{Access Control}

In an access control scenario, a requester sends an ODRL policy consisting of one or more permission rules that specify the intended actions by the requester over an asset. The evaluator must then take into consideration any internal ODRL policies and, potentially, a state of the world to decide if the request is granted. The access control scenario can be modelled as a monitoring scenario if we can extend the state of the world with a \emph{canonical event} that simulates the requested action having taken place. Intuitively, an access request should be granted if the state of the world preserves its validity after adding to it an event that matches that request.

\begin{definition}
    Given a timestamp $t$ and a rule $R$ where each constraint $c_i \in R$ is of the form $(\lambda_i = \rho_i)$, the canonical event at $t$ $e_R^t$ is the event $\langle t,\rho_1, \rho_2, \ldots, \rho_{n}\rangle$, where $n$ is the number of constraints in $R$.
\end{definition}

We can now define the access control decision problem in terms of the offline monitoring decision problem.

\begin{definition} \label{def:access_control}
    Given a rule $R$, a policy $\Pi$,
    a timestamp $t$ and a state of the world $\omega$, an access request for $R$ is accepted if the state of the world $\omega \cup \{e_R^t\}$ is valid w.r.t. $\Pi$.
\end{definition}

As an example, let the request $R$ be the rule $\{(Party=Alice),(Action=Read),(Asset=Book),(Pages=250)\}$ and let the policy $\Pi$ be the running example. Alice sends the request $R$ at ``2026-10-12T15:00:00'' to an endpoint with policy $\Pi$.
Let us assume first a case where the state of the world $\omega$ is empty. In such a case, the evaluation will check if the state of the world consisting of the single event $\langle\text{``2026-10-12T15:00:00''},Alice,Read,Book,250 \rangle$ is valid w.r.t. $\Pi$. Evidently, this is not true because the permission-duty pair $\langle p_{Read}, d_{Print} \rangle$ cannot be satisfied. Therefore, this request is not accepted.
On the other hand, if we assume the state of the world is the one shown in Table \ref{tab:sotw1} extended with $e_R^t$, it is valid w.r.t. $\Pi$. Therefore, this request is accepted. 

In the access control scenario, the usage of consequence, remedy and obligation rules is less natural. If a required duty of a permission is not fulfilled, a request to exercise that permission will be denied, thus never triggering its consequence rule. The same can be said of forbidden actions, that cannot trigger remedies as they cannot be violated. 
When we consider a policy that does not contain consequences, remedies or obligations, we can notice that the access control scenario can directly model the streaming monitoring scenario. 
\begin{theorem}
 Given a stream of events $\overline\omega$, such that every batch $\overline\omega_i \in \overline\omega$ contains a single event, and an ODRL policy $\Pi$ that contains only permission, prohibition and duty rules, $\overline\omega$ is valid w.r.t.\ $\Pi$ if and only if, for every batch $\overline\omega_i \in \overline\omega$, access request $R$ is accepted over state of the world $\bigcup_{j<i} \overline\omega_j$ at time $t$, where $t$ is the time stamp of the single event $e \in \overline\omega_i$, and $R$ is a permission that has $e$ as the canonical event.
\end{theorem}
\begin{proof}
$\Rightarrow$ If stream $\overline\omega$ is valid w.r.t\ $\Pi$, but access request $R$ of an event $e$ in $\overline\omega_i$ is denied over state of the world $\bigcup_{j<i} \overline\omega_j$, given the available rule types, it must be because 1) $e$ is not permitted, 2) $e$ is prohibited or 3) $e$ is permitted, but one of the associated duties is not matched to any event that predates $e$. However any of those three conditions would contradict the validity of $\overline\omega$ w.r.t\ $\Pi$.

$\Leftarrow$ If every access request $R$ is accepted, then by Definition \ref{def:access_control} $\{e\} \cup \bigcup_{j<i} \overline\omega_j$ is valid w.r.t. $\Pi$ for each event $e$ in $\overline\omega$. Then trivially, when considering the last event, the set of all events in $\overline\omega$ is valid w.r.t.\ $\Pi$ and thus the stream of events is valid w.r.t. $\Pi$ too by Definition \ref{def:streaming}.
\end{proof}

\section{The \ourevaluator Evaluator}

We now present the \ourevaluator evaluator and discuss its improvements against the state of the art. In particular, we describe a novel algorithm to process all types of ODRL rules using a single pass over the events of a state of the world, thus enabling scalable processing of all types of evaluation, including streaming.  
The main features of \ourevaluator are summarised in Table \ref{tab:featureComparison} along with a comparison with existing systems. \ourevaluator is the first ODRL Evaluator to be grounded in formal semantics, as it implements nearly all of the semantics defined in \cite{salas2025evaluation} (consequences of obligations being the only feature not supported yet). 
Note that the chosen semantics does not support set and membership constraint operators, whose mode of evaluation is not fully specified in ODRL. This grounding in formal semantics is a fundamental property, which makes the evaluation non-ambiguous. It enables precise testing, reuse and comparison with future systems, as the expected output of each policy evaluation is defined by the semantics.

{
\begin{table}[]
\newcommand{\rotate}[1]{\hspace{5px}\rotatebox{90}{#1}}
    \centering
     \caption{Features comparison of different ODRL evaluators.}
    \begin{tabular}{lcccccccccccccccc}
         System & \rotate{Permissions} & \rotate{Prohibitions} & \rotate{Obligations} & \rotate{Duties} & \rotate{Consequences} & \rotate{Remedies} & \rotate{Access Control} & \rotate{Monitoring} & \rotate{Reasoning} & \rotate{Constraints} & \rotate{Logical Constraints} & \rotate{Count} & \rotate{Time constraints} & \rotate{Ontologies} & \rotate{Stream Processing} & \rotate{Formal Semantics}  \\
         \midrule
         \ourevaluator & \checkmark & \checkmark & \checkmark & \checkmark & \checkmark & \checkmark & \checkmark & \checkmark &  & \checkmark & \checkmark & \checkmark & \checkmark & \checkmark & \checkmark & \checkmark \\
         \ghentEvaluator & \checkmark & \checkmark & \checkmark & & & & \checkmark & & \checkmark & \checkmark & \checkmark & & \checkmark & \checkmark\\
         ODRE Framework & \checkmark & \checkmark &  & & & & \checkmark & & & \checkmark & & & \checkmark & \checkmark  \\
         ODRL-PAP & \checkmark & & & & & & \checkmark & & & \checkmark & \checkmark & & \checkmark &  & \\
         ODRL-Manager & \checkmark & \checkmark & \checkmark & \checkmark & \checkmark & \checkmark & & & & \checkmark & \checkmark \\
         MOSAICrOWN & \checkmark & \checkmark &  & & & & \checkmark\\
    \end{tabular}
   
    \label{tab:featureComparison}
\end{table}}

As a comprehensive implementation of ODRL, \ourevaluator supports the basic rule types permissions, prohibitions and obligations, along with the more complex rules types of duties of permissions, consequences of permission duties, and prohibition remedies. It supports constraints and refinements. 
Notably, \ourevaluator is the first evaluator to fully support all of the ODRL logic constraints, on all rule types, including nested constraints. Nested constraints allow for the expression of complex logic expressions by combining the \emph{and}, \emph{or} and \emph{xor} logic operators. Time constraints are allowed through the \texttt{odrl:dateTime} left operand. Moreover, logs of events, or of exercised accesses requests are taken into account when evaluating the \texttt{odrl:count} constraint, which can be used to limit the number of times a permission can be exercised. \ourevaluator is domain-agnostic, as it supports ontologies and custom vocabularies by allowing any term, identified by an IRI, to be used as an action, party, asset or left operand.

The function that \ourevaluator uses to answer the \emph{evaluation} decision problem takes as input, at a minimum, an ODRL policy and a state of the world. In the implementation, a state of the world is assumed to be a CSV file that contains column names \texttt{odrl:Party}, \texttt{odrl:Action}, \texttt{odrl:Asset}, that map to the three main components of an ODRL rule, along with an additional column with the IRI of $\lambda$ as the header, for each left operand $\lambda$ found in the constraints of the policy, additionally prefixed with \texttt{odrl:Party}/\texttt{odrl:Action}/\texttt{odrl:Asset} and a space if it is the refinement of the party/action/asset component, respectively.

To support the streaming case of the monitoring scenario, as defined in Section \ref{sec:evaluatingScenarios}, an \ourevaluator evaluation outputs, along with the boolean result of the decision problem, an \emph{evaluation state} object, which records key pieces of information which can be recorded to be used when processing future batches of events. To illustrate the use of this object, imagine a policy containing a prohibition to access a file, subject to the remedy to make a payment to the file provider, a first batch of events that contain a record of the file being accessed, followed by a second batch of events that contain a record of a payment to the file provider. The evaluation on the first batch detects a violation, as the forbidden action has been performed, but it also outputs a state of the evaluation object that records the need to perform the remedial action. Based on this information, when processing the second batch the evaluator will consider the stream compliant, and output an updated state of the evaluation where the remedial action is no longer needed.

The \emph{evaluation state} object includes the information contained in the original policy used for the evaluation, and extends it in the following two ways. Firstly, each rule is annotated with the number of events previously matched against this rule ($matches\_count$) and whether the policy requires an event to match this rule in the future as a condition for validity ($required$). Then, a global field is recorded: $potential\_validity$, which is set to $0$ if the evaluation is invalid, and no future events can restore validity, for example if a prohibition has been violated, but no remedy for it exists. It is easy to see that the size of the evaluation state is constant with respect to the size of the policy. Therefore, for a policy of fixed size, and assuming all events are of fixed size, evaluating an additional event can be done in constant time, regardless of how many events have been evaluated previously, as all the inputs of the evaluation function remain constant in size. In a monitoring scenario, this property allows efficient processing of streams of events, and ensures that the complexity of evaluating of a log of events scales linearly with the size of the number of events, that is, with the size of the state of the world.

\begin{algorithm}[ht]
\begin{algorithmic}[1]
\Require row $x$, rule $R$, evaluation state $\Sigma$ (with $R$ being a rule in $\Sigma$)
\Function{CheckMatch}{$x, R, \Sigma$}
    \If{$x$ satisfies all conditions of $R$}
        \State $R.matches\_count \gets R.matches\_count + 1$
        \If{$R.required = 1$}
            \State $R.required \gets 0$
        \EndIf
        \State \Return True
    \EndIf
    \State \Return False
\EndFunction
\end{algorithmic}
\caption{Rule matching } 
\label{algorithm:helper}
\end{algorithm}

\begin{algorithm}[ht]
\begin{algorithmic}[1]
\Require Policy $\Pi$, state of the world $\mathcal{D}$, optional evaluation state $\Sigma$
\Function{Evaluate}{$\Pi, \mathcal{D}, \Sigma$}
    \If{$\Sigma$ not provided} \textit{//Initialise State of the Evaluation}
        \State initialise $\Sigma$ with $potential\_validity=1$; $counts$ and $required$ at 0 for rules
    \EndIf
    \State sort $\mathcal{D}$ by time ascending
    \State $validity \gets \Sigma.potential\_validity$
    \For{each row $x \in \mathcal{D}$}
        \State $M_P \gets \emptyset$, $M_F \gets \emptyset$
        \For{each permission $P$} \textit{// Check all rules once on this row}
            \State run \textproc{CheckMatch} on all duties of $P$ and their consequences
            \If{\Call{CheckMatch}{$x,P,\Sigma$}} add $P$ to $M_P$ \EndIf            
        \EndFor

        \For{each prohibition $F$}
            \State run \textproc{CheckMatch} on all remedies of $F$
            \If{\Call{CheckMatch}{$x,F,\Sigma$}} add $F$ to $M_F$ \EndIf
        \EndFor
        \For{each obligation $O$}
            \State \Call{CheckMatch}{$x,O,\Sigma$}
        \EndFor
        \If{$M_P = \emptyset$} \textit{// Detect violations of Permissions}
            \State $\Sigma$.potential\_validity $\gets$ $0$; $validity \gets 0$;
        \EndIf
        \For{each $P \in M_P$ and its duties $D$} \textit{// Enforce Duties and Consequences}
            \If{$D.matches\_count = 0$}
                \If{$D$ has no consequences}
                    \State $\Sigma$.potential\_validity $\gets$ $0$; $validity \gets 0$;
                \Else
                    \State set $D.required \gets 1$ and all consequences required
                \EndIf
            \EndIf
        \EndFor
        \For{each $F \in M_F$}
            \If{$F$ has no remedies} \textit{//Detect violations of Prohibitions}
                \State $\Sigma$.potential\_validity $\gets$ $0$; $validity \gets 0$;
            \Else  \textit{//Enforce Remedies}
                \State set all remedies required
            \EndIf
        \EndFor
    \EndFor
    \State check obligations: if never matched $\Rightarrow$ $validity \gets 0$
    \State check duties, consequences, remedies: if any still required  $\Rightarrow$ $validity \gets 0$
    \State \Return $validity$, $\Sigma$
\EndFunction
\end{algorithmic}
\caption{ODRL Evaluation Algorithm}
\label{algorithm:main}
\end{algorithm}

The algorithm of the evaluation function of \ourevaluator is described in Algorithm \ref{algorithm:main}. This algorithm evaluates a state of the world object in a single pass from the oldest to the newest event. As each event is evaluated against all rules, the helper function in Algorithm \ref{algorithm:helper} keeps track of how many times each rule matched an event through the $matches\_count$ property, and disables the $required$ flag of any matched rule which had such flag active. Intuitively, the $matches\_count$ property records how past events matched rules and it is used, for example, to check if a duty of a permission had been fulfilled. The $required$ flag instead defines the requirement of a rule to be matched by a subsequent event, for example to require a remedy to be eventually fulfilled. We provide an implementation 
of our evaluator along with a test suite of all of the core features of the evaluator, that contains many examples of policies, states of the world, and the expected results. This implementation also includes scripts to run the scalability tests presented next in Section \ref{sec:experiments}.

\section{Experimental Evaluation} \label{sec:experiments}

To evaluate the performance of our evaluator we have developed two generators capable of creating synthetic ODRL policies and states of the world based on a number of parameters. 

A synthetic ODRL policy is created by instantiating a given number of permission, prohibition and obligation rules. Permission rules can be configured to have one or more duties. Specific actions, parties, assets and left operands are selected randomly from a pool of 10 constants each, chosen at random for each policy, with realistic values taken from the ODRL vocabulary. 
Rules are created by first defining the action it applies to, and then with a 50\% chance each, specifying the asset and assignee components too.  Between 0 and 2 refinements/constraints are then assigned to each component and to the rule itself. A synthetic constraint is generated by randomly selecting a left operand, a numeric operator ($ <, \leq, =, \neq, >, \geq$) and an integer between 0 and 100 for the right operand.

A synthetic state of the world is generated from a policy by first extracting the set of features $\lambda$ that appear in any condition $(\lambda \ op\ \rho)$ in the rules of policy, and adding to it \texttt{odrl:dateTime} to ensure all events can be chronologically ordered. This set is then turned into a list to determine the columns of the state of the world table. A predefined number of rows is then created. Each row is generated by first selecting a random permission $P$, and then instantiating random values for each feature in the list, based on the conditions of $P$. 
If for a feature $\lambda$ there is no $(\lambda \ op\ \rho)$ in $P$, the field corresponding to that feature is left null with a 50\% chance (if it is not \texttt{odrl:Action}), or it is otherwise assigned one of 100 values randomly. Otherwise, one such $(\lambda \ op\ \rho)$ in $P$ is selected, and the field is populated with a value $v$ such that $v \; op\ \rho$ is true. For example, if $op$ is $>$ and  $\rho$ is $50$, a value satisfying this condition is created by adding a random number between $1$ and $100$ to $50$. 
This approach was chosen so that the generated rows will usually match existing permissions. While this does not guarantee validity, due to the complex ways the various randomly generated rules can interact with each other, it creates a more realistic test set by generating a significant amount of valid states of the world, alongside invalid ones.

\begin{figure}[ht!]
    \centering
    \small
    \begin{tikzpicture}[mark size = 1.5pt]
        \begin{groupplot}[
            group style={group size=2 by 2,  vertical sep = 1.0cm, horizontal sep = 1.5cm},
              width=.4\linewidth, height=.4\linewidth,
              grid=both,
              ylabel = Runtime (seconds),
              ymin = 0,
              ymax = 0.2,
            title style={at={(-0.3,1)},anchor=north west},
            legend style={nodes={scale=0.6, transform shape}}
            ]
            \nextgroupplot[
            xlabel = Policy Size,
            title={a)},
            legend style={at={(0,1)},anchor=north west}
            ]
            \addplot[only marks, mark = x, red] 
         table[col sep = comma, x = policy_size , y = runtime_seconds] {experiments/test_results/permissions_sotw_size_100_no_duties.csv};
         \addplot[red, mark = x] 
         table[col sep = comma, x = policy_size , y = runtime_seconds] {experiments/test_results/avg_permissions_sotw_size_100_no_duties.csv};
        
         \addplot[only marks, mark = o, blue] 
         table[col sep = comma, x = policy_size , y = runtime_seconds] {experiments/test_results/all_sotw_size_100_duties_remedies_all.csv};
         \addplot[blue, mark = o] 
         table[col sep = comma, x = policy_size , y = runtime_seconds] {experiments/test_results/avg_all_sotw_size_100_duties_remedies_all.csv};
         \legend{Permissions only,,Complex,,}
         
         \nextgroupplot[ xlabel = {Number of events}, 
         ylabel = {},
         xmin = 0,
         xmax = 1000,
         ymax = 1,
         title = {b)},
         legend style={at={(0,1)},anchor=north west}]
         \addplot[only marks, mark = x, red] 
         table[col sep = comma, x = sotw_size , y = runtime_seconds] {experiments/test_results_sotw/sotw_10_per_rule.csv};
         \addplot[red, mark = x] 
         table[col sep = comma, x = sotw_size , y = runtime_seconds] {experiments/test_results_sotw/avg_sotw_10_per_rule.csv};
         \addplot[only marks, mark = o, blue] 
         table[col sep = comma, x = sotw_size , y = runtime_seconds] {experiments/test_results_sotw/sotw_20_per_rule.csv};
         \addplot[blue, mark = o] 
         table[col sep = comma, x = sotw_size , y = runtime_seconds] {experiments/test_results_sotw/avg_sotw_20_per_rule.csv};
         \addplot[only marks, mark = square, black] 
         table[col sep = comma, x = sotw_size , y = runtime_seconds] {experiments/test_results_sotw/sotw_30_per_rule.csv};
         \addplot[black, mark = square] 
         table[col sep = comma, x = sotw_size , y = runtime_seconds] {experiments/test_results_sotw/avg_sotw_30_per_rule.csv};
         \legend{30 rules,,60 rules,,90 rules}
         \nextgroupplot[xmax = 11, 
         xlabel = {Duties per Permission},
         ymax = 1.7,
         title = {c)},
         legend style={at={(0,1)},anchor=north west}]
         \addplot[only marks, mark = x, red] 
         table[col sep = comma, x = duties_per_permission , y = runtime_seconds] {experiments/test_results_duties/duties_sotw_size_100.csv};
         \addplot[red, mark = x] 
         table[col sep = comma, x = duties_per_permission , y = runtime_seconds] {experiments/test_results_duties/avg_duties_sotw_size_100.csv};
         \addplot[only marks, mark = o, blue] 
         table[col sep = comma, x = duties_per_permission , y = runtime_seconds] {experiments/test_results_duties/duties_sotw_size_300.csv};
         \addplot[blue, mark = o] 
         table[col sep = comma, x = duties_per_permission , y = runtime_seconds] {experiments/test_results_duties/avg_duties_sotw_size_300.csv};
         \addplot[only marks, mark = square, black] 
         table[col sep = comma, x = duties_per_permission , y = runtime_seconds] {experiments/test_results_duties/duties_sotw_size_500.csv};
         \addplot[black, mark = square] 
         table[col sep = comma, x = duties_per_permission , y = runtime_seconds] {experiments/test_results_duties/avg_duties_sotw_size_500.csv};
         \legend{100 events,,300 events,,500 events}
         \nextgroupplot[xmax = 100, 
         xlabel = {Constraints per Rule},
         ylabel = {},
         ymax = 0.5,
         title = {d)},
         legend style={at={(0,1)},anchor=north west}]
         \addplot[only marks, mark = x, red] 
         table[col sep = comma, x = constraints_number_min , y = runtime_seconds] {experiments/test_results_constraints/constraints_sotw_size_100.csv};
         \addplot[red, mark = x] 
         table[col sep = comma, x = constraints_number_min , y = runtime_seconds] {experiments/test_results_constraints/avg_constraints_sotw_size_100.csv};
         \addplot[only marks, mark = o, blue] 
         table[col sep = comma, x = constraints_number_min , y = runtime_seconds] {experiments/test_results_constraints/constraints_sotw_size_300.csv};
         \addplot[blue, mark = o] 
         table[col sep = comma, x = constraints_number_min , y = runtime_seconds] {experiments/test_results_constraints/avg_constraints_sotw_size_300.csv};
         \addplot[only marks, mark = square, black] 
         table[col sep = comma, x = constraints_number_min , y = runtime_seconds] {experiments/test_results_constraints/constraints_sotw_size_500.csv};
         \addplot[black, mark = square] 
         table[col sep = comma, x = constraints_number_min , y = runtime_seconds] {experiments/test_results_constraints/avg_constraints_sotw_size_500.csv};
         \legend{100 events,,300 events,,500 events}
        \end{groupplot}
    \end{tikzpicture}
    \caption{Performance metrics of \ourevaluator}
    \label{fig:performance}
\end{figure}
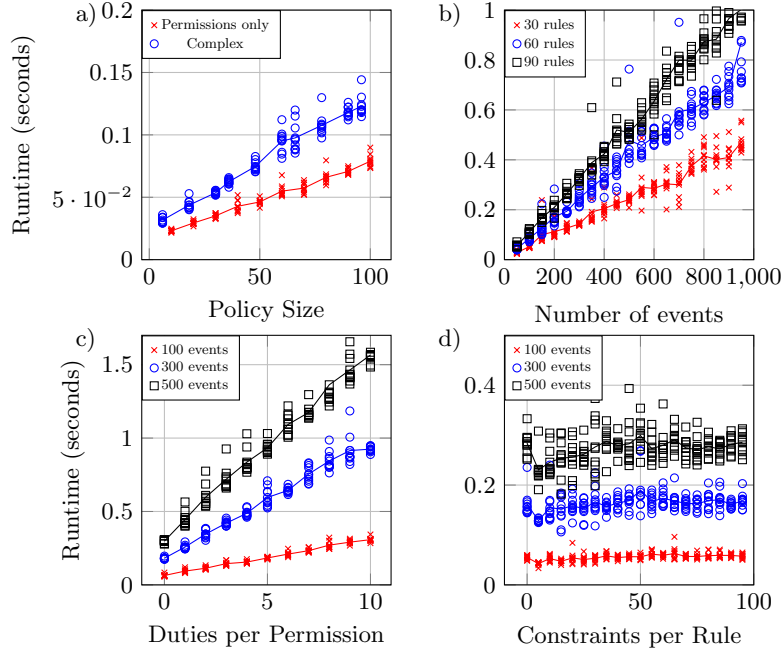

The results in Figure \ref{fig:performance} show the scalability tests we performed on \ourevaluator against different metrics, namely: the number of rules in the policy, the number of events in the state of the world, and the complexity of rules in terms of number of duties per permission, and constraints per rule. 
In Figure \ref{fig:performance}a), we examine the effects of the size and complexity of the policy on the runtimes. Here we can see that the runtimes increase linearly with respect to the size of the policy, and this behaviour doesn't change significantly even against policies with all rule types (complex policies). In Figure \ref{fig:performance}b), we examine the effect of the size of the state of the world (or number of events) on the runtimes. Once again, we can see a mostly linear behaviour, where the size of the policy alters the speed at which it increases. Then, in Figure \ref{fig:performance}c) we can also see that increasing the number of duties per permission increases the runtimes linearly. Finally, Figure \ref{fig:performance}d) shows a most notable result, where increasing the number of constraints per rule has almost no effect on the runtimes. This is most likely due to the early exit condition when checking multiple constraints, that terminates constraint evaluation after the first evaluation to false.

In order to run a comparison between \ourevaluator and the state-of-the-art \ghentEvaluator, we used the test cases found in the SolidLab ODRL Test Suite~\footnote{ODRL Test Suite: https://github.com/SolidLabResearch/ODRL-Test-Suite}. 

\begin{figure}[ht!]
    \centering
    \begin{tikzpicture}
        \begin{groupplot}[
            group style={group size=1 by 1,  vertical sep = 1.8cm, horizontal sep = 1.0cm},
            ybar,
            ymode = log,
            width=\linewidth,
            height=0.25\linewidth,
            ylabel={Runtime (ms)},
            xtick =data,
            xticklabel = \empty,
            enlargelimits=0.05,
            ymajorgrids=true,
            xmajorgrids=false,
            grid=none,
            grid style={solid, white!30},
            tick style={draw=none},
            legend style={at={(0.5,-0.1)},
            anchor=north,legend columns=-1},
            ybar interval=0.5,
            ],
            \nextgroupplot[]
            \addplot
            table[col sep = comma, x = id , y = runtime] {experiments/test_force/oval.csv};
            \addplot[red, fill=red!30]
            table[col sep = comma, x = id , y = runtime] {experiments/test_force/ghent.csv};
            \legend{\ourevaluator,\ghentEvaluator}
        \end{groupplot}
    \end{tikzpicture}
    \caption{Runtimes of test cases in the SolidLab Test Suite in logscale.}
    \label{fig:comparison1}
\end{figure}
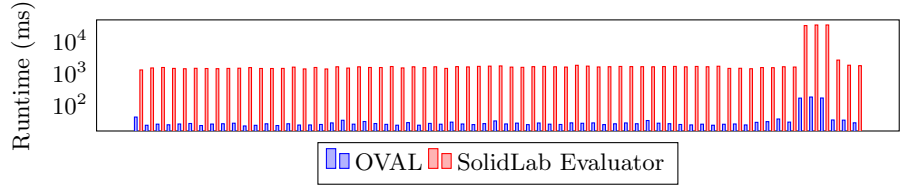

The results in Figure \ref{fig:comparison1} show a comparison of \ourevaluator against the \ghentEvaluator for the access control scenario against the test cases in the ODRL Test Suite. The runtimes are measured in milliseconds and shown in log scale because of the large difference. One can see that times are quite stable for both systems when the policies contain only permissions and constraints with datetimes, with an average of 30.32 ms for \ourevaluator and 2595.87 ms for the \ghentEvaluator.

However, there is a significant spike in three test cases, which are permissions with a large logical constraint (a disjunction of over 100 pairs of datetime constraints). For these cases, \ourevaluator takes slightly over 150 ms, whereas the \ghentEvaluator takes around 27 seconds. We note, however, that the \ghentEvaluator also performs reasoning as part of its evaluation, which may explain part of the significant overhead.

Regarding the actual result of the evaluation, the \ghentEvaluator has three possible values: allowed, not allowed or ``not enough information'', whereas \ourevaluator can only respond true or false. For the purpose of comparison, we assumed both ``not allowed'' and ``not enough information'' are false in our system (prohibited-by-default). With this in consideration, our evaluators agreed on 54 out of 68 test cases. Most discrepancies were either because the \ghentEvaluator performs reasoning while \ourevaluator does not, or because \ourevaluator supports duties while the \ghentEvaluator does not.
This suggests that the \ghentEvaluator may be more well-suited to use cases that require reasoning, whereas \ourevaluator for monitoring or other performance-sensitive scenarios. 

Regarding duties, the semantics adopted by \ourevaluator interprets the ODRL specification that a ``duty is fulfilled if all its constraints are satisfied and if its action, with all refinements satisfied, has been exercised''~\cite{iannella2018odrl}, as the existence of an event that matches the duty rule. This models in a natural way cases when a duty, such as the action of ``anonymisation'', allows certain future actions, such as ``sharing'', without setting a limit on how many times these actions can be performed. 
The \ghentEvaluator, instead, interprets duties in a stricter sense, by requiring \emph{duty reports} to map individual access requests with a specific fulfillments of a duty. This stricter interpretation models more naturally cases where each duty fulfilment enables an action to be performed just once, for example when each ``payment'' action allows a single ``seat reservation'' action. 

\section{Conclusion and Future Work}

This paper presented \ourevaluator, the first ODRL Evaluator grounded in a formal mathematical model of ODRL semantics. Moreover, we have formally defined the ODRL Evaluation problem for the access control, and offline and online monitoring, motivated by the different scenarios found in the literature where ODRL evaluation can be applied. We described the novel evaluation algorithm implemented by \ourevaluator, that supports all of the different evaluation problems in a scalable way. We have validated the efficiency of our system experimentally demonstrating linear complexity with respect to increases in both policy complexity and size of the state of the world it is evaluated on. The experiments in this paper focused mostly on the scalability of the system against different metrics, while correctness is backed up by formal semantics. Currently, our evaluator is being used in major EU projects with real-world use cases and industrial partners. 
Future work on the evaluator will focus on extending it with reasoning capabilities, supporting set operators in constraints, and adapting it to leverage existing semantic technologies such as the ODRL Compliance Report Model~\cite{esteves2025capturing}.

\paragraph*{Supplemental Material Statement:} Source code, documentation and test scripts are available on Github \footnote{\url{https://github.com/DIPS-Tools/odrl-Engine/}} and are archived on Zenodo \cite{paolo_2026_21415300}.

\bibliography{main}
\bibliographystyle{splncs04}

\end{document}